\documentclass{article}
\usepackage[preprint]{log_2022}

\usepackage{booktabs}           
\usepackage{multirow}           
\usepackage{amsfonts}           
\usepackage{amsthm}             
\usepackage{amsmath}            
\usepackage{amssymb}
\usepackage{microtype}
\usepackage{mathtools}          
\usepackage{graphicx}           
\usepackage{subcaption}         
\usepackage{xspace}             
\usepackage{enumitem}           
\usepackage{tikz}               
\usepackage[linesnumbered,noline,ruled]{algorithm2e}    
\usepackage[colorinlistoftodos,prependcaption, textsize=small%
]{todonotes} 
\usepackage[backref=page]{hyperref} 
\usepackage[noabbrev,nameinlink]{cleveref} 
\usepackage[numbers]{natbib} 
\usepackage{array}
\usepackage{adjustbox}
\usepackage{threeparttable}
\usepackage{comment}

\newtheorem{theorem}[algocf]{Theorem}
\newtheorem{lemma}[algocf]{Lemma}
\newtheorem{example}[algocf]{Example}
\theoremstyle{definition}
\newtheorem{definition}[algocf]{Definition}

\usetikzlibrary{shapes,fit}

\makeatletter 
    \tikzstyle{notestyleraw} = [
        draw=\@todonotes@currentbordercolor,
        fill=\@todonotes@currentbackgroundcolor,
        text=\@todonotes@currenttextcolor,
        rectangle,
        line width=0.5pt,
        text width = \@todonotes@textwidth - 1.6 ex - 1pt,
        rounded corners=4pt]
\makeatletter
 
\tikzset{every path/.style={very thick,shorten > =1pt, shorten < =1pt}, 
        every node/.style={circle,
        draw=rwth-blue, fill=rwth-llblue, stroke=thick, minimum size=20pt},
        every loop/.style={looseness=20}}

\newcommand\addvmargin[1]{
  \node[rectangle,fit=(current bounding box),inner ysep=#1,inner xsep=0, draw=none,fill=none]{};
}

\AtBeginEnvironment{algorithm}{}

\definecolor{rwth-blue}{cmyk}{1,.5,0,0}\colorlet{rwth-lblue}{rwth-blue!50}\colorlet{rwth-llblue}{rwth-blue!25}
\definecolor{rwth-violet}{cmyk}{.6,.6,0,0}\colorlet{rwth-lviolet}{rwth-violet!50}\colorlet{rwth-llviolet}{rwth-violet!25}
\definecolor{rwth-purple}{cmyk}{.7,1,.35,.15}\colorlet{rwth-lpurple}{rwth-purple!50}\colorlet{rwth-llpurple}{rwth-purple!25}
\definecolor{rwth-carmine}{cmyk}{.25,1,.7,.2}\colorlet{rwth-lcarmine}{rwth-carmine!50}\colorlet{rwth-llcarmine}{rwth-carmine!25}
\definecolor{rwth-red}{cmyk}{.15,1,1,0}\colorlet{rwth-lred}{rwth-red!50}\colorlet{rwth-llred}{rwth-red!25}
\definecolor{rwth-magenta}{cmyk}{0,1,.25,0}\colorlet{rwth-lmagenta}{rwth-magenta!50}\colorlet{rwth-llmagenta}{rwth-magenta!25}
\definecolor{rwth-orange}{cmyk}{0,.4,1,0}\colorlet{rwth-lorange}{rwth-orange!50}\colorlet{rwth-llorange}{rwth-orange!25}
\definecolor{rwth-yellow}{cmyk}{0,0,1,0}\colorlet{rwth-lyellow}{rwth-yellow!50}\colorlet{rwth-llyellow}{rwth-yellow!25}
\definecolor{rwth-grass}{cmyk}{.35,0,1,0}\colorlet{rwth-lgrass}{rwth-grass!50}\colorlet{rwth-llgrass}{rwth-grass!25}
\definecolor{rwth-green}{cmyk}{.7,0,1,0}\colorlet{rwth-lgreen}{rwth-green!50}\colorlet{rwth-llgreen}{rwth-green!25}
\definecolor{rwth-cyan}{cmyk}{1,0,.4,0}\colorlet{rwth-lcyan}{rwth-cyan!50}\colorlet{rwth-llcyan}{rwth-cyan!25}
\definecolor{rwth-teal}{cmyk}{1,.3,.5,.3}\colorlet{rwth-lteal}{rwth-teal!50}\colorlet{rwth-llteal}{rwth-teal!25}
\definecolor{rwth-gold}{cmyk}{.35,.46,.7,.35}
\definecolor{rwth-silver}{cmyk}{.39,.31,.32,.14}
\newcommand{\R}{\ensuremath{\mathbb{R}_{\neq0}}\xspace}

\newcommand{\N}{\ensuremath{\mathbb{N}}\xspace}
\newcommand{\B}{\ensuremath{\mathbb{B}}\xspace}

\renewcommand{\P}{\ensuremath{\textsf{P}}\xspace}

\renewcommand{\hom}{\textsf{\upshape hom}}
\newcommand{\Hom}{\textsf{\upshape Hom}}
\newcommand{\HomD}{\textsf{homD}}
\newcommand{\Emb}{\textsf{Emb}}
\newcommand{\emb}{\textsf{emb}}

\newcommand{\CC}{\mathcal{C}}
\newcommand{\CF}{\mathcal{F}}

\newcommand{\CK}{\mathcal{K}}
\newcommand{\CO}{\mathcal{O}}
\newcommand{\CP}{\mathcal{P}}
\newcommand{\CT}{\mathcal{T}}

\newcommand{\CCr}{\mathcal{C}^*}

\newcommand{\CKr}{\mathcal{K}^*}

\newcommand{\set}[1]{\left\{ #1 \right\}}

\SetKwProg{Fn}{\texttt{def}}{\string:}{}
\SetKwFunction{Range}{range}
\SetKw{KwTo}{in}\SetKwFor{For}{for}{\string:}{}%
\SetKwIF{If}{ElseIf}{Else}{if}{:}{elif}{else:}{}%
\SetKwFor{While}{while}{:}{}%

\SetKwInput{Input}{Input}
\SetKwInput{Output}{Output}
\SetKwData{Stack}{Stack} 
\SetKwData{Set}{Set}
\SetKwFunction{pop}{pop}
\SetKwFunction{F}{F}

\SetCommentSty{commentfont}
\SetKwComment{Comment}{$\#$\ }{}

\newcommand{\subheader}{\noindent\textbf}

\title[Structural Node Embeddings with Homomorphism Counts]{Structural Node Embeddings with Homomorphism Counts}

\author[H. Wolf et al.]
{Hinrikus Wolf\thanks{Equal Contribution}\,, Luca Oeljeklaus\footnotemark[1]\, \thanks{Funded by the European Union (ERC, SymSim,
101054974). Views and opinions expressed are however those of the
author(s) only and do not necessarily reflect those of the European
Union or the European Research Council. Neither the European Union
nor the granting authority can be held responsible for them.}\,, Pascal Kühner \and Martin Grohe\footnotemark[2] \\
\institute{RWTH Aachen University \\ Ahornstra\ss e 55, 52074 Aachen}\\
\email{\{hinrikus,oeljeklaus,grohe\}@informatik.rwth-aachen.de}}

\begin{document}

\maketitle

\begin{abstract}
Graph homomorphism counts, first explored by Lovász in 1967, have recently garnered interest as a powerful tool in graph-based machine learning.
Grohe (PODS 2020) 
proposed the theoretical foundations for using homomorphism counts in machine learning on graph level as well as node level tasks.
By their very nature, these capture local structural information, which enables the creation of robust structural embeddings.
While a first approach for graph level tasks has been made by Nguyen and Maehara (ICML 2020), we experimentally show the effectiveness of homomorphism count based node embeddings.
Enriched with node labels, node weights, and edge weights, these offer an interpretable representation of graph data, allowing for enhanced explainability of machine learning models.

We propose a theoretical framework for isomorphism-invariant homomorphism count based embeddings which lend themselves to a wide variety of downstream tasks.
Our approach capitalises on the efficient computability of graph homomorphism counts for bounded treewidth graph classes, rendering it a practical solution for real-world applications.
We demonstrate their expressivity through experiments on benchmark datasets.
Although our results do not match the accuracy of state-of-the-art neural architectures, they are comparable to other advanced graph learning models. 
Remarkably, our approach demarcates itself by ensuring explainability for each individual feature.
By integrating interpretable machine learning algorithms like SVMs or Random Forests, we establish a seamless, end-to-end explainable pipeline.
Our study contributes to the advancement of graph-based techniques that offer both performance and interpretability.
\end{abstract}

\section{Introduction}
 
Finding isomorphism invariant node embeddings consisting of meaningful and explainable features is one major task in graph learning. 
A powerful concept initially explored by \citeauthor{lovasz1967operations} in \citeyear{lovasz1967operations} is the theory of graph homomorphisms \cite{lovasz1967operations}.
Given two graphs \(G,H\), the basic question with respect to graph homomorphism counts is,

\begin{center}
        \emph{How many mappings are there from \(H\) to \(G\) such that \\ adjacent vertices in \(H\) are mapped to adjacent vertices in \(G\)?}
\end{center}

The quantification of homomorphisms from a graph \(H\) onto another graph \(G\) has proven to capture meaningful structural information of \(G\).
Counting homomorphisms from a family of graphs onto graphs \(G\) and \(G'\) then allows for a structural comparison of \(G\) and \(G'\).
Selecting appropriate graph classes, we leverage homomorphism counts to obtain meaningful and explainable embeddings for graphs as well as nodes.
By this notion, every feature has a clear origin based on the structure of the underlying graph.
Hence, these features are wholly explainable.
Further, homomorphism-based embeddings are isomorphism invariant, as \citeauthor{lovasz1967operations}' pioneering work shows that two graphs disagree on the homomorphism count for some graph \(H\) if and only if they are non-isomorphic.

The homomorphism counting problem from arbitrary graphs is known to be \(\textsf{\#}\P\)-complete \cite{grohe2020word2vec}.
However, there exist efficient algorithms to compute the homomorphism counts for graphs of bounded treewidth \cite{dalmau2004complexity}, with treewidth being a notion of how ``tree-like'' a graph is \cite{robertson1986graph}.

The theoretical interest in counting homomorphisms arises from the insight that many natural properties and invariants on graphs can be characterised in terms of homomorphism counts---these form a natural basis towards a rich feature space for graphs and are sometimes referred to as graph motif parameters \cite{curticapean2017homomorphisms}. 
Let us just highlight two results here: homomorphism counts from trees characterise a graph up to Weisfeiler-Leman indistinguishability \cite{dvorak2010recognizing}, or equivalently, indistinguishability by graph neural networks (GNNs) \cite{morris2019weisfeiler,xu2019powerful}. 
Homomorphism counts from cycles characterise graphs up to co-spectrality (that is, having the same multiset of eigenvalues). 
This is interesting, because it shows that counting homomorphisms extracts spectral information, which is well known to be relevant in machine learning. 
Hence, we focus on homomorphism counts from trees and cycles, which intuitively corresponds to extracting a combination of GNN and spectral features.

As a graph learning technique, following \citeauthor{grohe2020word2vec}'s proposal, homomorphism-based embeddings have been investigated on graph-level tasks \cite{nguyen2020graph, beaujean2021graph}.
We lift the theory of homomorphism counts to structural node embeddings.
To facilitate this extension, we enhance the nomenclature established by \citeauthor{grohe2020word2vec} \cite{grohe2020word2vec}, 
adapting it to accommodate (multi-)feature graphs.
For such graphs, we adapt algorithms for graph embeddings to node embeddings.
In pursuit of an explainable framework, we put forth an end-to-end interpretable pipeline for node classification.
This cohesive pipeline encompasses both the generation of structural node embeddings and the subsequent classification of nodes through random forests.
Employing both empirical (Cora, Citeseer, OBGN-Arxiv) and synthetic (DGL-Cluster, DGL-Pattern) datasets, we showcase the practical applicability of our developed embeddings.
The experiments do not show state-of-the-art performance, however they are on par with other classical node embeddings algorithms \cite{grover2016node2vec} and standard GNN architectures \cite{hu2020open}.

\section{Related Work}

The theoretical suitability of homomorphisms counts for machine learning has been repeatedly investigated \cite{dell2018lovasz, grohe2020word2vec, boeker2023fine}.
Experimental studies were performed on graph level tasks by computing \emph{homomorphism convolutions} \cite{nguyen2020graph} and sampling homomorphism densities \cite{beaujean2021graph}.
There are also applications where message passing GNNs are enriched with homomorphism counts of different sized cliques \cite{barcelo2021graph} and with sampled homomorphism counts \cite{welke2023expectation}. 
Similar approaches to counting homomorphisms, where small substructures of graphs were counted and then used for machine learning, have been investigated before.
Examples for these are the graphlet kernel \cite{shervashidze2009efficient}, motif-counting approaches \cite{ribeiro2022survey} and many more \cite{kriege2020survey}.
In \cite{zhao2023learned}, a framework is proposed that uses machine learning for counting subgraphs.
Input features for GNNs have been enriched with random walk \emph{structural encodings} \cite{dwivedi2022graph} which are similar to homomorphism counts of cycles.

\subheader{Theory of Homomorphism Counts.}
Graph homomorphism counts, known since the 1960s \cite{lovasz1967operations}, have emerged as a powerful tool for capturing graph properties. In particular, homomorphism indistinguishability, referring to two graphs admitting the same homomorphism counts over specific graph classes, has been shown to describe a variety of equivalence relations on graphs. Several results connect homomorphism counts to various fields, including logic \cite{dvorak2010recognizing,grohe2017descriptive,fluck2023going}, algebraic graph theory \cite{dell2018lovasz}, quantum information theory \cite{mancinska2020quantum}, category theory \cite{dawar2021lovasz}, and convex optimisation \cite{grohe2022homomorphism,roberson2023lasserre}.

\subheader{Node embeddings and other approaches to graph learning.}
In recent years, quite a few node embedding algorithms have been proposed.
There exist approaches based on random walks e.g. node2vec \cite{grover2016node2vec}, spectral methods like HOPE \cite{ou2016asymmetric} as well as neural architectures such as SDNE \cite{wang2016structural} and GraphSage \cite{hamilton2017inductive}.
Over the years, several extensions to these algorithms have been proposed \cite{zhou2022network}. 
Node level tasks have been also tackled by various GNN architectures.
Graph Convolutional Neural Networks \cite{kipf2017semisupervised} as well as the more general family of message passing neural networks \cite{gilmer2017neural} have been applied to node classification, for example by the OGB project \cite{hu2020open}.
The basic methods have been extended to networks like GIN \cite{xu2019powerful}, GAT \cite{brody2022attentive} and more recently by incorporating transformer networks, e.g., GPS \cite{rampasek2022recipe}.
More architectures and standardised benchmarks have been collected and compared by the OGB project \cite{hu2020open}.

\section{Preliminaries}

\label{sec:preliminaries}
We denote the set of reals by $\mathbb R$ and the set of non-zero reals by $\R$.
We denote tuples \((x_1, \cdots, x_k)\) as \(\bar{x}\) and matrices as capital letters \(A,B, \dots\). We access elements by \(\bar{x}_i\) resp.~\(A_{ij}\) for appropriate indices.
Given $k_1,k_2 \in \N$ with \(k_1 \le k_2\), we write \([k_1,k_2]\coloneqq (k_1, k_1 +1, \dots, k_2)\) and abbreviate \([k_2] \coloneqq [1,k_2]\). 

We assume graphs to be undirected and denote the node and edge set of a graph \(G\) by \(V(G)\) and \(E(G)\), respectively. We usually assume that \(V(G)=[n]\), where \(n=|V(G)|\) is the \emph{order} of \(G\). 
Then, we can define the \emph{adjacency matrix} of \(G\) to be the matrix \(A^G \in \mathbb{R}^{n \times n}\) with entries \(A^G_{vw}=1\) if \((v,w)\in E(G)\) and \(A^G_{vw}=0\) otherwise. 
The nodes \(v\in V(G)\) may have \emph{node features} \(f_G(v)\in\mathbb R^d\), where \(d\ge 1\) is the \emph{feature dimension}. Both in theory and practice it will be important for us to only allow non-zero features or even positive features. 
In practice, we replace all zeroes by a small constant. 
We will discuss this issue in \Cref{sec:method} again. The default feature value will be \(1\). In particular, in a \emph{plain graph} without node features we assume that the feature dimension is \(1\) and \(f_G(v)=1\) for all \(v\in V(G)\).

A \emph{rooted graph} is a pair \((G,r)\) where \(G\) is a graph and \(r\in V(G)\) a vertex (called the \emph{root}).
We will consider various graph classes, most importantly the class 
\(\CC\) of \emph{cycles}, the class \(\CP\) of all \emph{paths}, and  the class \(\CT\) of all \emph{trees}. Many of the theoretical results extend to the class \(\CT_k\) of all graphs of tree width \(k\), for \(k\ge 1\) (see \cite{dell2018lovasz,dvorak2010recognizing}). 
We do not need this extension here, but is worth noting that \(\CC,\CP,\CT\subseteq\CT_2\), that is, cycles, paths and trees all have tree width at most \(2\).
We regard the classes \(\CC,\CP,\CT,\CT_k\) as classes of plain graphs. 
For a class \(\CK\) of plain graphs, we denote by \(\CK^\circ\) the class of all featured graphs with underlying plain graph in \(\CK\). 
By \(\CK^*\) and \(\CK^{\circledast}\) we denote the class of all rooted graphs \((G,v)\) with \(G\in\CK\) and \(G\in\CK^\circ\), respectively.

Our approach readily generalises to graphs which also have edge features (or weights), but for simplicity we focus on graphs with only node features here.

\section{Theoretical Foundation: Counting Homomorphisms}

\label{sec:theory}

A \emph{homomorphism} from a graph \(F\) to a graph \(G\) is a mapping \(\mu:V(F)\to V(G)\) that preserves edges, that is, \(\big(\mu(v),\mu(w)\big)\in E(G)\) for all \((v,w)\in E(F)\). 
To capture information on the local structure of the graph around a specific vertex \(v\), this definition can be extended naturally to rooted graphs: a \emph{homomorphism} from a rooted graph \((F,r)\) to a rooted graph \((G,v)\) is a homomorphism \(\mu\) from \(G\) to \(H\) with \(\mu(r)=v\). Since we are mainly interested in node embeddings and structural encodings in this paper, we focus on the rooted case in the following.

To review the basic theory, it is convenient to first ignore node features and focus on plain graphs (where \(f_G(v)=1\) for all \(v\in V(G)\)). 
We let \(\Hom(F,r,G,v)\) be the set of all homomorphisms from \((F,r)\) to \((G,v)\) and \(\hom(F,r,G,v)\coloneqq|\Hom(F,r,G,v)|\) the number of homomorphisms.
It is our goal to gather information about a graph \((G,v)\) from the homomorphism numbers \(\hom(F,r,G,v)\) for graphs \((F,r)\) from some family \(\CF\). In practice, the families \(\CF\) we consider are finite families of trees and cycles. In theory, they are usually infinite families. We state rooted versions of the most important theoretical results. These results are usually stated in the unrooted cases, but can easily be adapted (see \cite{grohe2020word2vec}).

\begin{theorem}[\cite{lovasz1967operations}]
\label{theo:graphs_isomorphism}
    For all connected graphs \(G,H\) and nodes \(v \in V(G), w \in V(H)\), the following are equivalent.
    \begin{enumerate}
        \item For all rooted graphs \((F,r)\), it holds that \(\hom(F,r,G,v) = \hom(F,r,H,w)\).
        \item There is an isomorphism \(\gamma: V(G) \to V(H)\) such that \(\gamma(v) = w\).
    \end{enumerate}
    In fact, if \(|V(G)|, |V(H)| \le n\), it suffices to consider graphs \(F\) of order at most \(n\).
\end{theorem}

  The following theorem refers to the Weisfeiler-Leman algorithm, which by now is well known in the graph learning community, and also includes its well-known relation to graph neural networks. We refer the reader to \cite{grohe2020word2vec,morris2021weisfeiler} for background.

\pagebreak

\begin{theorem}[\cite{dvorak2010recognizing,morris2019weisfeiler,xu2019powerful}]
\label{theo:weisfeiler_leman_homomorphisms}
    For all connected graphs \(G,H\) and nodes \(v \in V(G), w \in V(H)\), the following are equivalent.
    \begin{enumerate}
        \item For all rooted trees \((T,r)\in\CT^*\), it holds that \(\hom(T,r,G,v) = \hom(T,r,H,w)\).
        \item \((G,v),(H,w)\) cannot be distinguished by the \(1\)-dimensional Weisfeiler-Leman algorithm.
          \item \((G,v),(H,w)\) cannot be distinguished by a graph neural network.
    \end{enumerate}
    In fact, if \(|V(G)|, |V(H)| \le n\), it suffices to consider trees \(T\) of height at most \(n\).
 \end{theorem}

A generalisation of this theorem to graphs of treewidth \(k\), the \(k\)-dimensional Weisfeiler-Leman algorithm, and higher-order graph neural networks can be found in \cite{dvorak2010recognizing,morris2019weisfeiler}. 
Bounded treewidth graphs are of particular relevance here, as counting homomorphisms from a class \(\CK\) of graphs is in a polynomial time if and only if the class has bounded treewidth (under some complexity theoretic assumptions) \cite{dalmau2004complexity}. 
Thus if we want to compute \(\hom(F,r,G,v)\) efficiently, we should ensure that \(F\) has small treewidth.

A third result relevant to us relates homomorphism counts to spectral graph theory, specifically homomorphism counts from cycles to co-spectrality. The rooted version of this result is a little awkward, but let us state it anyway.

\begin{theorem}[Folklore]
    For all graphs \(G,H\), the following are equivalent:
    \begin{enumerate}
        \item \(G\), \(H\) are co-spectral, that is, their adjacency matrices have the same multisets of eigenvalues.
        \item For all rooted cycles \((C,r) \in \CCr\), it holds that \begin{center}
            \(\sum_{v \in V(G)} \hom(C,r,G,v) = \sum_{v \in V(H)} \hom(C,r,H,v)\).
        \end{center} 
    \end{enumerate}
\end{theorem}

All these results hold for plain graphs. Let us extend the definitions and results to weighted graphs. Let \((F,r)\), \((G,v)\) be graphs of feature dimension \(d\). Then we define \(\hom(F,r,G,v)\in\mathbb R^d\) to be the vector with entries
\[
  \hom(F,r,G,v)_i \coloneqq \displaystyle\sum_{\mu \in \operatorname{Hom}(F,r,G,v)} \; \displaystyle\prod_{x \in V(F)} f_F(x)_i\cdot f_G(\mu(x))_i
\]
for \(i\in[d]\). Note that this is consistent with the previous definition for plain graphs: if \(d=1\) and \(f_F(x)=f_G(y)=1\) for all \(x\in V(F),y\in V(G)\) then \(\displaystyle\prod_{x \in V(F)} f_F(x)_i\cdot f_G(\mu(x))_i=1\), and \(\hom(F,r,G,v)\) is just the number of homomorphisms.

Remarkably, the fundamental \Cref{theo:graphs_isomorphism} does not extend from plain graphs to graphs with nontrivial features, as the following simple example (also illustrated in \Cref{fig:weighted_homomorphisms_break_lovasz}) shows. 

\begin{example} \label{ex:weighted_homomorphisms}
  Let \(G\) be the tree of height \(1\) with root \(\textcolor{rwth-carmine}{v_0}\) and two leaves \(v_1,v_2\). Let \(f_G\) be the 1-dimensional feature map defined by \(f_G(\textcolor{rwth-carmine}{v_0})\) and \(f_G(v_1)=f_G(v_2)=\frac{1}{2}\). Let \(H\) be the tree of height \(1\) with root \(\textcolor{rwth-lcarmine}{w_0}\) and one leaf \(w_1\) and features \(f_H(\textcolor{rwth-lcarmine}{w_0})=f_H(w_1)=1\). It is not hard to prove that for all rooted graphs \((F,r)\) it holds that \(\hom(F,r,G,\textcolor{rwth-carmine}{v_0})=\hom(F,r,H,\textcolor{rwth-lcarmine}{w_0})\).
\end{example}

\begin{figure}
    \centering
    \begin{tikzpicture}
            \node[label={left:\(f_G: 1\)}] (0) at (0,1.2) {\textcolor{rwth-carmine}{\(v_0\)}};
            \node[label={left:\(f_G: \frac{1}{2}\)}] (1) at (-1,0) {};
            \node[label={right:\(f_G: \frac{1}{2}\)}] (2) at (1,0) {};
        
            \node[label={right:\(f_H: 1\)}] (3) at (5,1.2) {\textcolor{rwth-lcarmine}{\(w_0\)}};
            \node[label={right:\(f_H: 1\)}] (4) at (5,0) {};

            \draw[-] (0) to[] (1);
            \draw[-] (0) to[] (2);
            \draw[-] (3) to[] (4);

        \end{tikzpicture}
    \caption{\Cref{ex:weighted_homomorphisms}, illustrated. It is readily verified that, for any rooted graph \((F,r)\), it holds that \(\hom(F,r,G,\textcolor{rwth-carmine}{v_0}) = \hom(F,r,H,\textcolor{rwth-lcarmine}{w_0})\). Nevertheless it holds that \(G \not \cong H\), contradicting the idea that \Cref{theo:graphs_isomorphism} might extend to weighted homomorphism counts.}
    \label{fig:weighted_homomorphisms_break_lovasz}
\end{figure}

There is a generalisation of the first part of \Cref{theo:weisfeiler_leman_homomorphisms}, the equivalence between homomorphism counts over trees and the Weisfeiler-Leman algorithm, but it requires an adaptation of the WL-algorithm to the weighted setting: instead of refining by the numbers of nodes of each colour, we need to refine by a weighted sum.

\section{Methodology}

\label{sec:method}

To obtain a homomorphism count embedding of a graph \(G\), the usual procedure consists in counting the element-wise number of homomorphisms from a fixed family of graphs \(\CK\) to \(G\), thus yielding an embedding of \(G\) into a latent space \(\mathbb{R}^{|\CK|}\).
To adapt this technique for node embeddings, we consider the graph family to consist of rooted graphs. 
For each graph therein, we map the root to the node we want to embed and count the number of homomorphisms that can still be realised.
As left-hand graphs, we consider graph families of small treewidth, in particular: rooted (binary) trees, rooted cycles and rooted paths.
As discussed above, it is possible to consider other graph classes, the only proviso being a bound on the treewidth.

\subsection{Embeddings}

\begin{figure}
    \centering
        \begin{subfigure}{0.29\textwidth}
            \begin{tikzpicture}
            \node[draw=none, fill=none] (dummy) at (0,-1.5) {};
            \node[] (a) at (0,0) {\(\textcolor{rwth-lcarmine}{v}\)};
            \node[] (b) at (1.5,0) {\(\phantom{a}\)};
            \node[] (c) at (3,0) {\(\textcolor{rwth-carmine}{v'}\)};
            \node[] (d) at (0,-1.5) {\(\phantom{a}\)};
            \node[] (e) at (1.5,-1.5) {\(\phantom{a}\)};
            \node[] (f) at (3,-1.5) {\(\phantom{a}\)};
            \node[] (g) at (3,1.5) {\(\phantom{a}\)};

            \draw[-] (a) to[] (b);
            \draw[-] (a) to[] (d);
            \draw[-] (b) to[] (c);
            \draw[-] (b) to[] (d);
            \draw[-] (b) to[] (e);
            \draw[-] (b) to[] (g);
            \draw[-] (c) to[] (e);
            \draw[-] (e) to[] (f);
        \end{tikzpicture}
        \caption{A graph \(G\) and two highlighted vertices \(\textcolor{rwth-lcarmine}{v}, \textcolor{rwth-carmine}{v'} \in V(G)\) .}
        \end{subfigure}
        \begin{subfigure}{0.70\textwidth}
            \begin{tabular}{ccc}
            \toprule
                \(H\) & \(\Hom(H,r,G,\textcolor{rwth-lcarmine}{v})\) & \(\Hom(H,r,G,\textcolor{rwth-carmine}{v'})\) \\ \midrule
                \begin{tikzpicture}[baseline=0]
                    \node[] (a) at (0,-0.4) {\(\phantom{a}\)};
                    \node[] (b) at (1.5,-0.4) {\(\phantom{a}\)};
                    \node[] (c) at (0.75,0.4) {\(r\)};
                    \draw[-] (a) to[] (b);
                    \draw[-] (a) to[] (c);
                    \draw[-] (b) to[] (c);
                    \addvmargin{1mm}
                \end{tikzpicture} & \(\textcolor{rwth-lcarmine}{2}\) & \(\textcolor{rwth-carmine}{2}\) \\ \midrule
                \begin{tikzpicture}[baseline=0]
                    \node[] (a) at (0,0) {\(r\)};
                    \node[] (b) at (1.25,0) {\(\phantom{a}\)};
                    \node[] (c) at (2.5,0) {\(\phantom{a}\)};
                    \draw[-] (a) to[] (b);
                    \draw[-] (b) to[] (c);
                    \addvmargin{1mm}
                \end{tikzpicture} & \(\textcolor{rwth-lcarmine}{7}\) & \(\textcolor{rwth-carmine}{8}\) \\
                \bottomrule
            \end{tabular}
            \begin{minipage}{.0cm}
            \vfill
            \end{minipage}
            \caption{Two rooted graphs \((C_3,r)\) and \((P_3,r)\). Both \(\textcolor{rwth-lcarmine}{v}\) and \(\textcolor{rwth-carmine}{v'}\) admit two rooted homomorphisms from \((C_3,r)\), but disagree on \((P_3,r)\).}
        \end{subfigure}
    \caption{Embedding of two vertices \(\textcolor{rwth-lcarmine}{v}\) and \(\textcolor{rwth-carmine}{v'}\) from a graph \(G\) over the set of rooted left-hand graphs \(\set{(C_3,r),(P_3,r)}\). Their embedding vectors are given by \(\emb^G_{((C_3,r),(P_3,r))}(\textcolor{rwth-lcarmine}{v}) = (\textcolor{rwth-lcarmine}{2},\textcolor{rwth-lcarmine}{7})\) and \(\emb^G_{((C_3,r),(P_3,r))}(\textcolor{rwth-carmine}{v'}) = (\textcolor{rwth-carmine}{2},\textcolor{rwth-carmine}{8})\), respectively.} 
    \label{fig:homomorphism_embedding}
\end{figure}

Given a family of rooted graphs \(\CKr\) and a graph \(G\), we denote an \(\CKr\)\emph{-homomorphism embedding} of a node \(v \in V(G)\) as \(\emb^G_{\CKr}(v) \in \mathbb{R}^{|\CKr|}\), which we define as
\[\emb^G_{\CKr}(v) \coloneqq (\Hom(H, r, G, v))_{(H,r) \in \CKr},\] 
that is, for every \((H,r) \in \CKr\), the number of homomorphisms from \((H,r)\) to \(G\) where the root \(r\) is mapped to \(v\).
This can be seen in \Cref{fig:homomorphism_embedding}.
In particular, we take into account node features as described in \Cref{sec:theory}.
Over all nodes of a graph \(G\), this yields an embedding \(\Emb^G_{\CKr} \in \mathbb{R}^{|V(G)| \times |\CKr|}\).

\begin{definition}[Homomorphism tensor embeddings]
    If \(G\) is a multi-featured graph \(G = (V(G),E(G),f^1_G,\dots,f^m_G)\), we adapt the procedure as follows.
    Compute, for every \(v \in V(G)\) and every \(G_1 \coloneqq (V(G),E(G),f^1_G), \dots, G_m \coloneqq (V(G),E(G),f^m_G)\), the embeddings \(\emb^{G_i}_{\CKr}(v)\) with \(i \in [m]\).
    The \emph{homomorphism tensor embedding} \(\emb^{G}_{\CKr}(v) \in \mathbb{R}^{|\CKr| \cdot m}\) is then defined as
    \[\emb^{G}_{\CKr}(v) \coloneqq \left(\emb^{G_1}_{\CKr}(v), \dots, \emb^{G_m}_{\CKr}(v)\right),\]
    which comes down to computing one set of weighted homomorphism counts for every feature.
\end{definition}
When using categorical features, zero-valued features have to be preprocessed, since a multiplication by \(0\) would result in information loss.
As discussed in \Cref{sec:preliminaries} a solution is to disturb zero-entries by some value \(\varepsilon\) close to zero.
Empirically, we discovered that \(\varepsilon = 0.01\) is a good candidate for this.

The overall embedding of the nodes of a graph is then of the form 
\(\Emb^G_{\CKr} \in \mathbb{R}^{|V(G)| \times (|\CKr| \cdot m)}\).
As left-hand graphs, we specifically consider these  graph classes:
\begin{itemize}
    \item[--] Trees of order at most 12. For each such tree, we select one vertex as the root.
    \item[--] Binary trees of order at most 12. These admit a unique vertex which works as a root.
    \item[--] Cycles of order at most 10. The selection of the root is unimportant.
    \item[--] Paths of order at most 10. We select as root an endpoint of the path.
\end{itemize}

\subheader{Variations.}
Through basic combinatorial arguments, it is not hard to see that homomorphism counts tend to grow exponentially on most reasonable left-hand graph classes.
This can pose challenges for certain machine learning algorithms like SVMs and neural networks, as they are sensitive to vector elements spanning a wide range of magnitudes.
Further, these numbers can outgrow the scope of fixed precision integer arithmetic.
A sound remedy to this exponential blow-up can be to scale the numbers logarithmically.
In preliminary experiments using logarithmically scaled homomorphism counts, we found a reasonable increase in classification accuracy when using SVMs.

Another solution to the exponential blow-up we explored is the notion of homomorphism densities.
Given two graphs \(G,H\) and vertices \(v \in V(G), r \in V(H)\), the \emph{homomorphism density} \(\HomD(H,r,G,v)\) is defined as
\[\HomD(H,r,G,v) \coloneqq \frac{\Hom(H,r,G,v)}{(|H|-1)^{|G|-1}}.\]
This corresponds to the share of mappings \(\mu\) from \(V(H)\) to \(V(G)\) maintaining \(\mu(r) = v\) which are homomorphisms.
To that end, \citeauthor{beaujean2021graph} \cite{beaujean2021graph} proposed to sample a number of mappings between graphs and verify for each whether it constitutes a homomorphism or not. 
They then obtain guarantees on the quality of the approximation through standard Chernoff bound arguments.
It is possible to adapt this homomorphism density sampling technique to a node-level setting, by fixing one element of the mapping.
However, this proved to be prohibitive due to the high number of samples required to obtain a reasonable approximation on the homomorphism density.

The embeddings can be enriched by adding vector based node features, using multiple families of graphs, or using homomorphism tensors.
Another variation is to compute the homomorphism counts of a whole graph without fixing a node.
This variation is suitable for graph-level tasks.
The application of this leaves the scope of this paper, however we performed preparatory experiments.
The results of these can be found in \Cref{apx:graph-tasks}.

\subsection{Algorithms}

In \Cref{lem:weighted_cycle} and \Cref{lem:weighted_path} we present the procedures we use to computed rooted cycle and path homomorphisms. \Cref{algo:compute_tree_homomorphisms} to compute tree homomorphism can be found in \Cref{apx:algo}.

\begin{lemma}[Weighted Cycle Homomorphism Counts]
\label{lem:weighted_cycle}
    Let \(G\) be a graph with adjacency matrix \(A^G\). Further, let \(W \in \mathbb{R}^{n \times n}\) be the feature matrix defined as
    \begin{align*}
        W_{vu} = \begin{cases}
            w_G(v) & \text{if } v = u, \\
            0 & \text{otherwise.}
        \end{cases}
    \end{align*}
    Then letting, \(M = (A^G \cdot W)\), we obtain for all \(v \in V(G)\) and \(k \ge 2\) that
    \[\Hom(C_k,1,G,v) = M^k_{vv}.\]
\end{lemma}
\begin{proof}
    This follows from standard graph theoretic arguments, see e.g. \cite[section 6.10]{newman2010networks}.
\end{proof}
The same procedure can be applied to paths by using the row-sums as results. However, as matrix-matrix multiplication is notoriously computationally expensive, we work around this by, essentially, pulling the row-sums into the parentheses.

\begin{lemma}[Weighted Path Homomorphism Counts]
\label{lem:weighted_path}
    Let \(G\) be a graph with adjacency matrix \(A^G\). Further, let \(\bar{w}_v \in \R^{n}\) be the feature vector defined as \(\bar{w}_v = w_G(v)\) for all \(v \in V(G)\).
    Then, for all \(k \in \N\) and \(v \in V(G)\):
    \[\Hom(P_{k},1,G,v) = 
    \begin{cases}
        (A^G \cdot (\Hom(P_{k-1},1,G,u))_{u \in V(G)}) \cdot \bar{w}_v & \text{ if } k>1, \\
        \bar{w}_v &  \text{ if } k=1, \\
        \bar{0} &  \text{ if } k=0. \\
    \end{cases}\]
\end{lemma}

Letting \(k\) be the size of left-hand graph we consider, we shortly discuss the computational complexity of our implementation.
Computing weighted cycle homomorphism counts comes down to computing matrix powers, which yields a complexity of \(\CO(n^\omega\log(k))\), where \(\omega < 2.373\) \cite{alman2021refined}.
Computing weighted path homomorphism counts, in turn, has complexity \(\CO(kn^2)\).
Computing weighted tree homomorphism counts, using \Cref{algo:compute_tree_homomorphisms}, also has complexity \(\CO(k n^2)\) \cite{diaz2002counting}.


\section{Experiments}
\subsection{Setup and Implementation}

To demonstrate the theoretical concept of using weighted homomorphism counts for node embeddings, we conduct experiments on five standard benchmark datasets for node classification.
In the following, let \(n,m\), and \(c\) stand for the number of nodes, edges, and node classes of a dataset, respectively. 

\begin{description}
    \item[Cora \cite{mccallum2000automating}] is a citation network with \(n = 2708, m = 5429, c= 7\). Each node is equipped with a binary feature vector \(\bar{i} \in \B^{1433}\) corresponding to the occurrence (or not) of some word in the corresponding publication.
    \item[Citeseer \cite{giles1998citesser}] is a citation network with \(n = 3312, m = 4732, c= 6\). Each node is equipped with a binary feature vector \(\bar{i} \in \B^{3703}\) corresponding to the occurrence (or not) of some word in the corresponding publication.
    \item[OGBN-Arxiv \cite{hu2020open}] is a citation network with \(n = 169343, m = 1166243, c= 40\). Each node is equipped with a feature vector \(\bar{i} \in \R^{128}\) corresponding to a word2vec-embedding of the corresponding article's title and abstract. 
    \item[DGL-Cluster \cite{dwivedi2023benchmarking}] is a set of 12.000 graphs, each with \(n \in [40, 190], c = 6\) and containing 6 stochastic block model clusters which have to be detected. Each graph has a node feature which takes value 0 except for one node per community, which is assigned a value \(x \in [1,6]\).
    \item[DGL-Pattern \cite{dwivedi2023benchmarking}] is a set of 14.000 graphs, each with \(n \in [44, 188],c = 2\) and containing 5 communities created from stochastic block models with \(p = 0.5, q = 0.35\), as well as a ``pattern'' to be detected, that is, a sixth community with \(p = 0.5, q = 0.5\).
\end{description}
We chose Cora and Citeseer as they are reasonably small and thus suitable for our embeddings.
OGBN-Arxiv was selected as the smallest node property dataset from the OGB project, the standard benchmark for graph learning tasks.
Finally, the DGL datasets enable a comparison with \cite{barcelo2021graph}.

For each dataset, we obtain a range of node embeddings by computing homomorphism counts on families of rooted paths (\texttt{Paths}), cycles (\texttt{Cycles}), binary trees (\texttt{Binary Trees}), and trees (\texttt{Trees}) as left-hand graphs as well as the corresponding homomorphism tensors (\texttt{Tensor / T.}).
Additionally, we performed experiments in which we augment the embedding vectors by the raw node features (\texttt{Features}).
We evaluated the embeddings both individually as well as in the form of ensembles; we performed downstream tasks on the concatenation of multiple embeddings. For example, {\texttt{T. Trees + Cycles}} corresponds to performing an experiment over the node embedding consisting of the tensor tree homomorphism counts, and cycle homomorphism counts.

We evaluate our embeddings on node classification tasks. 
For OBGN-Arxiv, there exists a standard performance evaluator \cite{hu2020open}, which we use according to the guidelines; we perform 10 repetitions with different random seeds on a random forest classifier and report the averaged accuracy and standard deviation.
Note that OGBN-Arxiv has positive as well as negative features, which can lead to vanishing numbers; practically, this was unlikely due to the word2vec-embedding and it did indeed not occur.
For DGL-Cluster and DGL-Pattern, we concatenated the embeddings of the individual graphs and then used the standard experimental setup \cite{dwivedi2023benchmarking}; we also trained a random forest analogously to our own setup on the (\(K_3, K_4, K_5\))-homomorphisms counts provided in \cite{barcelo2021graph}.
On Cora and Citeseer, we perform 10-fold stratified cross-validation and report the averaged accuracy and standard deviation.
The implementation was done in python using \texttt{networkx}, \texttt{numpy}, and \texttt{scipy}. 
We use the random forest classifier implementation from \texttt{scikit-learn} with default parameters.

\subsection{Results}

\Cref{tab:main_result} and \Cref{tab:some_dgl_results} present the outcomes of our experimental trials alongside benchmark outcomes from existing literature. The best results from our experiments are indicated in bold typeface.
Our results show that the incorporation of node features significantly increases performance.
Notably, embeddings founded on cycle-based methods yield consistently strong outcomes. 
On Cora, \texttt{Tensor Cycles} displays the best performance, while on Citeseer, \texttt{Paths + Features} yields the best results.
However, the differences among the best results all fall within their respective standard deviations, inhibiting the identification of a clear best candidate.

Ensembles of embeddings do not enhance overall accuracy (see \Cref{tab:results_all} in \Cref{app:all}). 
On Cora and Citeseer, our outcomes are competitive with other advanced node embedding techniques.
Conversely, the OGBN-Arxiv dataset reveals subpar downstream performance across all embeddings.
Additionally, cycle-based embeddings encountered computational timeouts after 48 hours.
On DGL-Cluster and DGL-Pattern, we see mixed results.
While DGL-Pattern is not far off the neural baseline, our results on DGL-Cluster are lacking. 
Remarkably, the (\(K_3, K_4, K_5\))-homomorphism embedding \cite{barcelo2021graph} performs very poorly without a powerful GNN architecture.
The expressiveness of feature-inclusive embeddings remains evident, yet we lag behind in comparison to the baseline.

\begin{table}
\caption{Avg. accuracy of node classification over select embeddings (excerpt of \Cref{tab:results_all}).}
\label{tab:main_result}
\centering
\begin{threeparttable}
\begin{tabular}{>{\ttfamily}lccccc}

\toprule
\textrm{Embedding} / \textrm{Dataset} &             CITESEER &                 CORA &           OGBN-Arxiv \\
\midrule
Features                                        &  \(0.762 \pm 0.054\) &  \(0.779 \pm 0.025\) &  \(0.455 \pm 0.001\) \\
Binary Trees                                    &  \(0.497 \pm 0.086\) &  \(0.601 \pm 0.031\) &  \(0.315 \pm 0.001\) \\
Tensor Binary Trees                           &  \(0.753 \pm 0.063\) &  \(0.837 \pm 0.012\) &  \(\mathbf{0.551 \pm 0.001}\) \\
Cycles                                          &  \(0.461 \pm 0.078\) &  \(0.525 \pm 0.042\) &                 \tnote{*}  \\
Tensor Cycles                                 &  \(0.781 \pm 0.058\) &  \(\mathbf{0.859 \pm 0.018}\) &                  \tnote{*} \\
Paths                                           &  \(0.527 \pm 0.092\) &  \(0.597 \pm 0.024\) &    \(0.500 \pm 0.001\)  \\
Paths + Features                                &  \(\mathbf{0.788 \pm 0.061}\) &  \(0.792 \pm 0.026\) &                  \(0.501 \pm 0.001\)   \\
\midrule
GCN \cite{kipf2017semisupervised, hu2020open} &  \(0.679 \pm 0.005\) & \(0.801 \pm 0.005\) & \(0.717 \pm 0.030\) \\
GAT (+FLAG) \cite{velickovic2018graph, kong2022robust} & \(0.725 \pm 0.007\) & \(0.830 \pm 0.005\) & \(0.737 \pm 0.001\) \\
SSP \cite{izadi2020optimization} &\(0.805 \pm 0.001\) & \(0.902 \pm 0.006\) & -\\
node2vec \cite{hu2020open} & - & - & \(0.701 \pm 0.001\) \\
SimTeG+TAPE+RevGAT \cite{duan2023simteg} & - & - & \(0.780 \pm 0.000\) \\

\bottomrule
\end{tabular}
    \begin{tablenotes}
    \item[*] Cycle embeddings run into a timeout after 48 hours.
    \end{tablenotes}
\end{threeparttable}
\end{table}
\begin{table}[b]
    \centering
    \caption{Avg. weighted accuracy of node classification over select embeddings (excerpt of \Cref{tab:dgl_results}).}
    \label{tab:some_dgl_results}
    \begin{tabular}{>{\ttfamily}lcc}
\toprule
\textrm{Embedding} / \textrm{Dataset} & DGL-CLUSTER & DGL-PATTERN \\
\midrule
Binary Trees                & \(0.207 \pm 0.051\) & \(0.777 \pm 0.034\) \\
Binary Trees + Features     & \(0.209 \pm 0.033\) & \(0.777 \pm 0.011\) \\
Tensor Cycles               & \(\mathbf{0.459 \pm 0.053}\) & \(0.765 \pm 0.013\) \\
\midrule
Binary Trees + Cycles       & \(0.166 \pm 0.037\) & \(0.799 \pm 0.022\) \\
Binary Trees + T. Cycles    & \(0.444 \pm 0.095\) & \(\mathbf{0.800 \pm 0.017}\) \\
\midrule
GatedGCN+PE (16 L.)  \cite{dwivedi2023benchmarking}        & \(0.761 \pm 0.020\) & \(0.868 \pm 0.009\) \\
GatedGCN+PE + \(K_3, K_4,K_5\)  \cite{barcelo2021graph}    & \(0.740 \pm 0.019\) & \(0.856 \pm 0.033\) \\
RF on \(K_3, K_4,K_5\)                                     & \(0.143 \pm 0.001\) & \(0.334 \pm 0.005\) \\
\bottomrule
\end{tabular}
\end{table}

We provide the embedding computation times for each class/dataset combination in \Cref{tab:run-times} in \Cref{app:all}.
We show only the computation times with-out features, since these computations add up in the number of features.
However, the number of features will dominate this computation.
All experiments have been performed on an AMD EPYC 7302 processor.

\section{Discussion}
We present a theoretical framework for gaining meaningful and explainable node embeddings based on homomorphism counts.
These excel in capturing the structure of the underlying graphs.
The inclusion of features through tensor embeddings or the appending of feature vectors further enhances their expressivity.
Especially the effect of the tensor embeddings on the results is remarkable.

On the benchmark datasets, we see mixed performance by our models.
When comparing to the raw node features as baseline, it becomes apparent that the datasets vary in how much information resides in the structure.
Whereas Citeseer embeddings only showed marginal downstream improvement when homomorphism embeddings are used, Cora embeddings show drastically improved downstream accuracy on \texttt{Tensor Cycles} compared to the raw features.
Unfortunately, OGBN-Arxiv --even though it is the smallest graph from the OGB project for node classification-- is too large to compute high-performing embeddings.
This could be improved by increasing dimension, i.e., counting larger paths, trees, and so on.
However, preliminary experiments on smaller datasets suggest that increasing the size of the left-handed graphs has only marginal effects.
Independently, because of the sheer size of that graph, it becomes unfeasible to compute larger embeddings.
Another potential explanation could be that the random forest classifier struggled with the number of parameters.

Concerning the generated datasets, DGL-Pattern carries a lot of information within its structure; the reverse is the case for DGL-Cluster.
In DGL-Pattern, the ``pattern'' distinguishes itself in a structural way from the remaining graph, and the node features are just random noise. 
Hence, the results from simple and tensor embeddings differ only minimally and both are quite performant.
In DGL-Cluster on the other hand, the information contained in the structure is almost non-existent, as the communities are not distinguished in terms of structure. 
As such, the prediction accuracy based on simple homomorphism embeddings is close to the ratio of number of classes.
In contrast, the tensor embeddings are much more expressive, as they are able to capture information on the location of the node in question in terms of the 6 labelled nodes.
For the citation networks, it was in fact not clear at all that the local structure around a node should carry any meaningful information.
Contradicting the intuition, our results clearly underscore the presence of such information in the structure of these datasets and the importance of capturing it.

The embedding ensembles, in particular the combination of tree and cycle embeddings, did not generally improve accuracy.
The idea was to combine tree features (1-WL) with cycle features (2-WL, spectral information).
A boosting based on this was only observable on DGL-Pattern.

The tensor embeddings we introduce, while allowing us to capture more meaningful information, lead to a substantial overhead in our procedure, as we compute homomorphism counts for each label.
This effect is especially significant on the empirical datasets, since their nodes are equipped with sizeable feature vectors. 
To address this issue, we tried preprocessing the node features by means of a PCA, yielding a reduced set of labels for homomorphism count computation.
However, this decreased downstream performance.
Another idea, concerning tree embeddings, would be to sample trees up to a larger size to capture a wider radius, at the cost of potentially losing granularity. 

All in all, we find our homomorphism count based node embeddings capture a considerable amount of meaningful structural information; the introduction of tensor embeddings allows us to combine it with node-level information.


\section{Conclusion}

Graph homomorphism counts are a powerful tool for explainable node embeddings.
By their very nature, homomorphism counts extract local graph structure based on the left-hand graphs.
Unlike simple message passing neural networks, they are able to capture structural information beyond 1-WL.
While our embeddings do not show state-of-the-art downstream performance, they are reasonably competitive.
However, especially for large graphs, their computational costs are quite high.
Hence, we propose that homomorphism count based node embeddings on smaller graphs can complement the available structural encodings for state-of-the-art neural architectures.

Our main contribution is the introduction of node-level weighted homomorphism counts and subsequent tensor embeddings, which form a novel approach to include multi-featured homomorphism counts into node embeddings.
Combined with explainable learning algorithms, our setup allows for an end-to-end explainable machine learning framework for node-based tasks.
Therefore, in scenarios where interpretability is paramount, our methodology proves to be advantageous and well-suited.

The purpose of this paper is to demonstrate that homomorphism counts from a small family of simple graphs can extract meaningful features that may provide useful structural encodings for the nodes of a graph. In future work, we plan to port our algorithms to a GPU-based parallelised implementation, exploiting the matrix and tensor-based algorithms for counting homomorphisms from cycles, paths, and full binary trees.

\bibliographystyle{plainnat}
\bibliography{reference}

\newpage
\appendix

\section{Algorithm}
\label{apx:algo}

We compute rooted tree homomorphisms counts using \Cref{algo:compute_tree_homomorphisms}, an adaption from \cite{diaz2002counting}. On input of a tree \(T,r\), a graph \(G\) and a node feature \(w: V(G) \to \R\), it computes \(\hom(T,r,G,v)\) in parallel for all \(v \in V(G)\).
In the following we try and provide an intuitive description of the case where \(w(v) = 1\) for all \(v \in V(G)\).
For every vertex of the tree, it initialises an array of length \(|V(G)|\) with value 1 in every entry (the number of homomorphisms from the singleton graph to a node).
The algorithm then works along a postordering of the tree vertices.
When visiting a node \(t_n\), the entries of that node's parent node (\(t_p\)) are updated. Once all the children of a node \(t_p\) have been visited, \(t_p\)'s final vector entries correspond to the number of rooted homomorphisms of the subtree rooted in \(t_p\) to every vertex of \(G\).
The weighted case takes into account the node weights as necessary.

\begin{algorithm}[h!]
    \DontPrintSemicolon
    \caption{Weighted tree homomorphisms for every vertex of \(G\).}
    \label{algo:compute_tree_homomorphisms}
    \SetFuncSty{textbf}
    \SetKwFunction{rth}{rooted\_tree\_homomorphisms}
    \SetKwFunction{th}{tree\_homomorphisms}
    \SetKwFunction{dfs}{dfs\_postorder\_nodes}
    \SetKwFunction{random}{randomElement}
    \SetKwFunction{sumover}{sum}
    \Fn{\rth{\(T,r,G,w\)}}{
    \Input{\begin{tabular}{rl}
         \((T,r)\):& a rooted tree,  \\
         \(G\):&  a graph, \\
         \(w\):& a vertex feature.
    \end{tabular}}%
    \Output{\((\Hom(T,r,G,v))_{v \in V(G)}\)}%
    \BlankLine
    \For{\(t\) \KwTo \(V(T)\)}{%
        \(\text{predecessor}_t \coloneqq u \in N_T(t) \text{ with } \textbf{dist}_T(r,u) < \textbf{dist}_T(r,t)\) \;
        \(\bar{H}_t \coloneqq (\begin{array}{ccc}
         w(1) & \cdots & w(n) \\
    \end{array})\)\Comment*{This corresponds to the homomorphism counts from the rooted singleton graph for each vertex of \(G\)}
    }
    \For{\(t\) \KwTo \dfs{\(T,r\)} \(- \; r\)}
        {
        \(p \coloneqq \text{parent}_t\) \;
        \(\bar{U} \coloneqq A^G \cdot (\bar{H}_t \cdot w_{V(G)})\) \;
        \(\bar{H}_p \coloneqq \bar{H}_p \cdot \bar{U}\)
        }
    \Return \(\bar{H}_r\)\Comment*{Contains one value for every vertex in \(V(G)\)}
    }
    \BlankLine
    \Fn{\th{\(T,G,w\)}}{
    \(r\) = \random{\(V(G)\)} \;
    \Return \sumover{\rth{\(T,r,G,w\)}}
    }
\end{algorithm}

\newpage

\pagebreak
\section{All Results}
\label{app:all}
\begin{table}[h!]
\caption{Avg. accuracy of node classification over all embeddings.}
\label{tab:results_all}
\centering
\begin{threeparttable}
\begin{tabular}{>{\ttfamily}lccc}

\toprule
\textrm{Dataset} &             CITESEER &                 CORA &           OGBN-Arxiv \\
\textrm{Embedding}                                      &                      &                      &                      \\
\midrule
Features                                        &  \(0.762 \pm 0.054\) &  \(0.779 \pm 0.025\) &  \(0.455 \pm 0.001\) \\
Binary Trees                                    &  \(0.497 \pm 0.086\) &  \(0.601 \pm 0.031\) &  \(0.315 \pm 0.001\) \\
Binary Trees + Features                         &  \(0.727 \pm 0.054\) &  \(0.709 \pm 0.015\) &  \(0.489 \pm 0.001\) \\
Tensor Binary Trees                           &  \(0.753 \pm 0.063\) &  \(0.837 \pm 0.012\) &  \(\mathbf{0.551 \pm 0.001}\) \\
Cycles                                          &  \(0.461 \pm 0.078\) &  \(0.525 \pm 0.042\) &                 \tnote{*}  \\
Cycles + Features                               &   \(0.785 \pm 0.050\) &  \(0.779 \pm 0.021\) &                 \tnote{*}  \\
Tensor Cycles                                 &  \(0.781 \pm 0.058\) &  \(\mathbf{0.859 \pm 0.018}\) &                  \tnote{*} \\
 Trees                            &  \(0.491 \pm 0.086\) &  \(0.615 \pm 0.032\) &  \(0.307 \pm 0.001\) \\

 Trees + Features                 &  \(0.695 \pm 0.054\) &  \(0.683 \pm 0.025\) &                  \(0.473 \pm 0.001\) \\
 Tensor Trees                   &  \(0.751 \pm 0.065\) &  \(0.841 \pm 0.012\) &  \(0.546 \pm 0.001\) \\
Paths                                           &  \(0.527 \pm 0.092\) &  \(0.597 \pm 0.024\) &    \(0.500 \pm 0.001\)  \\
Paths + Features                                &  \(\mathbf{0.788 \pm 0.061}\) &  \(0.792 \pm 0.026\) &                  \(0.501 \pm 0.001\)   \\
Tensor Paths                                  &  \(0.785 \pm 0.053\) &   \(0.845 \pm 0.02\)  &  \(0.521 \pm 0.001\)  \\
\midrule
Binary Trees + Cycles                           &  \(0.511 \pm 0.087\) &  \(0.604 \pm 0.028\) &                  \tnote{*}  \\
Binary Trees + T. Cycles                  &  \(0.778 \pm 0.059\) &  \(0.855 \pm 0.017\) &                  \tnote{*} \\
T. Binary Trees + Cycles                  &  \(0.756 \pm 0.064\) &  \(0.841 \pm 0.012\) &                  \tnote{*} \\
T. Binary Trees + T. Cycles         &  \(0.761 \pm 0.064\) &  \(0.841 \pm 0.011\) &                  \tnote{*} \\

 Trees + Cycles                   &  \(0.506 \pm 0.088\) &  \(0.616 \pm 0.023\) &                  \tnote{*}  \\
 T. Trees + Cycles          &  \(0.776 \pm 0.061\) &  \(0.853 \pm 0.019\) &                  \tnote{*}  \\
 T. Trees + Cycles          &  \(0.749 \pm 0.056\) &  \(0.837 \pm 0.015\) &                  \tnote{*}  \\
 T. Trees + T. Cycles &  \(0.756 \pm 0.064\) &  \(0.842 \pm 0.015\) &                  \tnote{*}  \\
 \midrule
GCN \cite{kipf2017semisupervised, hu2020open} &  \(0.679 \pm 0.005\) & \(0.801 \pm 0.005\) & \(0.717 \pm 0.030\) \\
GAT (+FLAG) \cite{velickovic2018graph, kong2022robust} & \(0.725 \pm 0.007\) & \(0.830 \pm 0.005\) & \(0.737 \pm 0.001\) \\
SSP \cite{izadi2020optimization} &\(0.805 \pm 0.001\) & \(0.902 \pm 0.006\) & -\\
node2vec \cite{hu2020open} & - & - & \(0.701 \pm 0.001\) \\
SimTeG+TAPE+RevGAT \cite{duan2023simteg} & - & - & \(0.780 \pm 0.000\) \\

\bottomrule
\end{tabular}
    \begin{tablenotes}
    \item[*] Cycle embeddings run into a timeout after 48 hours.
    \end{tablenotes}
\end{threeparttable}
\end{table}
\begin{table}
    \centering
    \caption{Avg. weighted accuracy of node classification over all embeddings.}
    \label{tab:dgl_results}
    \begin{tabular}{>{\ttfamily}lcc}
\toprule
\textrm{Dataset} & DGL-CLUSTER & DGL-PATTERN \\
\textrm{Embedding} &  &  \\
\midrule
Binary Trees                & \(0.207 \pm 0.051\) & \(0.777 \pm 0.034\) \\
Tensor Binary Trees         & \(0.345 \pm 0.035\) & \(0.778 \pm 0.037\) \\
Binary Trees + Features     & \(0.209 \pm 0.033\) & \(0.777 \pm 0.011\) \\
Cycles                      & \(0.210 \pm 0.048\) & \(0.764 \pm 0.028\) \\
Tensor Cycles               & \(\mathbf{0.459 \pm 0.053}\) & \(0.765 \pm 0.013\) \\
Cycles + Features           & \(0.210 \pm 0.084\) & \(0.763 \pm 0.015\) \\
Trees                       & \(0.199 \pm 0.045\) & \(0.778 \pm 0.029\) \\
Tensor Trees                & \(0.351 \pm 0.018\) & \(0.780 \pm 0.024\) \\
Trees + Features            & \(0.208 \pm 0.057\) & \(0.778 \pm 0.028\) \\
Paths                       & \(0.209 \pm 0.018\) & \(0.772 \pm 0.007\) \\
Tensor Paths                & \(0.414 \pm 0.085\) & \(0.776 \pm 0.022\) \\
Paths + Features            & \(0.209 \pm 0.050\) & \(0.772 \pm 0.037\) \\
\midrule
Binary Trees + Cycles       & \(0.166 \pm 0.037\) & \(0.799 \pm 0.022\) \\
Binary Trees + T. Cycles    & \(0.444 \pm 0.095\) & \(\mathbf{0.800 \pm 0.017}\) \\
T. Binary Trees + Cycles    & \(0.348 \pm 0.040\) & \(0.793 \pm 0.034\) \\
T. Binary Trees + T. Cycles & \(0.426 \pm 0.070\) & \(0.798 \pm 0.047\) \\
Trees + Cycles              & \(0.167 \pm 0.097\) & \(0.795 \pm 0.024\) \\
Trees + T. Cycles           & \(0.430 \pm 0.036\) & \(0.799 \pm 0.013\) \\
T. Trees + Cycles           & \(0.353 \pm 0.062\) & \(0.790 \pm 0.040\) \\
T. Trees + T. Cycles        & \(0.413 \pm 0.041\) & \(0.798 \pm 0.021\) \\
\midrule
MLP (4 L.) \cite{dwivedi2023benchmarking}                  & \(0.210 \pm 0.000\) & \(0.505 \pm 0.000\) \\
Vanilla GCN (4 L.) \cite{dwivedi2023benchmarking}          & \(0.534 \pm 0.203\) & \(0.639 \pm 0.007\) \\
GCN (16 L.)   \cite{dwivedi2023benchmarking}               & \(0.690 \pm 0.137\) & \(0.856 \pm 0.003\) \\
GatedGCN (16 L.)  \cite{dwivedi2023benchmarking}           & \(0.738 \pm 0.033\) & \(0.856 \pm 0.009\) \\
GatedGCN+PE (16 L.)  \cite{dwivedi2023benchmarking}        & \(0.761 \pm 0.020\) & \(0.868 \pm 0.009\) \\
GCN + \(K_3, K_4,K_5\) \cite{barcelo2021graph}             & \(0.630 \pm 0.372\) & \(0.825 \pm 0.048\) \\
MoNet \cite{barcelo2021graph}                              & \(0.712 \pm 0.033\) & \(0.859 \pm 0.003\) \\
MoNet + \(K_3, K_4,K_5\)  \cite{barcelo2021graph}          & \(0.723 \pm 0.036\) & \(0.866 \pm 0.003\) \\
GatedGCN+PE + \(K_3, K_4,K_5\)  \cite{barcelo2021graph}    & \(0.740 \pm 0.019\) & \(0.856 \pm 0.033\) \\
RF on \(K_3, K_4,K_5\)                                     & \(0.143 \pm 0.001\) & \(0.334 \pm 0.005\) \\
\bottomrule
\end{tabular}
\end{table}

\begin{table}[ht]
    \centering
        \caption{Computation times of featureless embeddings per dataset in seconds (unless specified otherwise).}
    \begin{tabular}{>{\ttfamily}llllll}
        \toprule
        \textrm{Embedding / Dataset} & Citesser & Cora & OGBN-Arxiv & DGL-Pattern & DGL-Cluster \\
        \midrule
        Paths & 0.012 & 0.014 & 5.091 & 27.35 & 16.88\\
        Cycles & 0.532 & 1.508 & TIMEOUT & 135.1 & 92.4\\
        Binary Trees & 2.621 & 2.203  & 23 min &  6 h& 2 h\\
        Trees & 9.003 & 7.644 & \(\approx \) 40 h& 21 h & 6 h\\
        \bottomrule
    \end{tabular}

    \label{tab:run-times}
\end{table}

\pagebreak
\section{Results on Graph Embedding}
\label{apx:graph-tasks}

Here, we present some further results on experiments with homomorphism count based graph embeddings.
Due to these not constituting the main focus of our paper, we present them in the appendix to showcase that our models can easily be adapted for graph level embeddings.
We use an analogous setup as that used for node embeddings with the sole difference being that we do not compute rooted homomorphisms and that the classifier we use is an SVM instead of a random forest.
We provide an overview of the datasets used in \Cref{tab:datasets}.
 %
In \Cref{tab:ogbg-molhiv}, we present the performance of OGBG-Molhiv from the OGB project.
In \Cref{tab:comparison_graph}, we compare the performance of our embeddings on examples from the TUdataset \cite{morris2020tudataset} with the performance of multiple Weisfeiler-Leman kernels obtained in \cite{nikolentzos2021graph}.

\begin{table}[h]
    \caption{Datasets used for graph classification tasks.}
    \centering
    \begin{tabular}{lccccc}\toprule
        Dataset & Graphs & Classes & \(\varnothing\)Num. Nodes & \(\varnothing\)Num. Edges & Node Features \\ \midrule
        DD \cite{morris2020tudataset} & 1178 & 2 & 284.32 & 715.66 & Yes \\
        ENZYMES \cite{morris2020tudataset} & 600 & 6 & 32.63 &  62.14 & Yes \\
        PROTEINS \cite{morris2020tudataset} & 1113 & 2 & 39.06 & 72.82 & Yes \\
        MUTAG \cite{morris2020tudataset} & 188 & 2 & 17.93 & 19.79 & Yes \\
        NC I1 \cite{morris2020tudataset} & 4110 & 2 & 29.87 & 32.30 & Yes \\
        IMDB-Binary \cite{morris2020tudataset} & 1000 & 2 & 19.77 & 96.53 & No \\
        IMDB-Multi \cite{morris2020tudataset} & 1500 & 3 & 13.00 & 65.94 & No \\
        OGBG-Molhiv \cite{hu2020open} & 41127 & 2 & 25.50 & 27.50 & Yes \\ \bottomrule
    \end{tabular}
    \label{tab:datasets}
\end{table}

\begin{table}[hb!]
    \caption{Results of Experiments done on OGBG-Molhiv.}
    \centering
    \begin{tabular}{>{\ttfamily\small}lc} \toprule
        \textrm{Dataset} & OGBG-Molhiv \\
        \textrm{Embedding} &  \\ \midrule
        Cycles & \(0.705 \pm 0.00\) \\
        Tensor Cycles & \(\mathbf{0.747 \pm 0.01}\) \\
        Paths &\(0.710 \pm 0.01\) \\
        Tensor Paths & \(0.732 \pm 0.01\) \\       
        Trees & \(0.686 \pm 0.01\) \\
        Tensor Trees & \(0.738 \pm 0.01\) \\ \midrule
        Cycles + Paths & \(0.718 \pm 0.01\) \\
        T. Cycles + T. Paths & \(0.741 \pm 0.01\) \\
        Cycles + Trees & \(0.710 \pm 0.01\) \\
        T. Cycles + T. Trees & \(0.743 \pm 0.01\) \\ \midrule 
        GCN \cite{dwivedi2023benchmarking} & \(0.759 \pm 0.12\) \\
        GIN \cite{dwivedi2023benchmarking} & \(0.771 \pm 0.15\) \\
        GCN+hom+F \cite{welke2023expectation} & \(0.788 \pm 0.01\)\\
        GIN+hom+F \cite{welke2023expectation} & \(0.777 \pm 0.02 \)\\
        \bottomrule
    \end{tabular}
    
    \label{tab:ogbg-molhiv}
\end{table}

\begin{table}[ht!]
    \centering
    \caption{Comparing our homomorphism count based graph embeddings with multiple Weisfeiler-Leman Kernels \cite{nikolentzos2021graph} on graph classification.}
    \begin{adjustbox}{angle=90}
    \begin{tabular}{>{\ttfamily\small}lccccccc} \toprule
         \textrm{Dataset} & DD & ENZYMES & MUTAG & NCI1 & PROTEINS & IMDB-BINARY & IMDB-MULTI \\
\textrm{Embedding / Kernel} & & & & & & & \\\midrule
Paths & \(0.766 \pm 0.03\) & \(0.265 \pm 0.05\) & \(0.670 \pm 0.02\) & \(0.744 \pm 0.04\) & \(0.888 \pm 0.06\) & \(\mathbf{0.719 \pm 0.04}\) & \(0.494 \pm 0.04\) \\
Tensor Paths & \(0.787 \pm 0.03\) & \(\mathbf{0.422 \pm 0.06}\) & \(0.790 \pm 0.02\) & \(0.757 \pm 0.04\) & \(0.889 \pm 0.06\) & - & - \\
Binary Trees & \(0.767 \pm 0.03\) & \(0.266 \pm 0.05\) & \(0.716 \pm 0.02\) & \(0.744 \pm 0.04\) & \(0.888 \pm 0.06\) & \(0.717 \pm 0.04\) & \(0.493 \pm 0.03\) \\
Tensor Binary Trees & \(0.787 \pm 0.03\) & \(0.412 \pm 0.07\) & \(0.752 \pm 0.02\) & \(0.761 \pm 0.04\) & \(0.884 \pm 0.06\) & - & - \\
Trees & \(0.766 \pm 0.03\) & \(0.277 \pm 0.05\) & \(0.701 \pm 0.02\) & \(0.744 \pm 0.04\) & \(0.890 \pm 0.06\) & \(0.717 \pm 0.04\) & \(\mathbf{0.495 \pm 0.03}\) \\
Tensor Trees & \(0.787 \pm 0.03\) & \(0.412 \pm 0.07\) & \(0.755 \pm 0.02\) & \(\mathbf{0.765 \pm 0.04}\) &\(0.892 \pm 0.06\) & - & - \\
Cycles & \(0.768 \pm 0.03\) & \(0.311 \pm 0.05\) & \(0.712 \pm 0.02\) & \(0.747 \pm 0.040\) & \(0.888 \pm 0.06\) & \(0.713 \pm 0.05\) & \(0.492 \pm 0.03\) \\
Tensor Cycles & \(\mathbf{0.788 \pm 0.03}\) & \(0.414 \pm 0.06\) & \(\mathbf{0.804 \pm 0.02}\) & \(0.771 \pm 0.04\) & \(\mathbf{0.895 \pm 0.06}\) & - & - \\ 
\midrule
Cycles + Paths & \(0.768 \pm 0.03\) & \(0.344 \pm 0.06\) & \(0.727 \pm 0.02\) & \(0.744 \pm 0.04\) & \(0.888 \pm 0.06\) & \(0.722 \pm 0.04\) & \(0.491 \pm 0.04\) \\
W. Cycles + W. Paths & \(0.788 \pm 0.03\) & \(0.469 \pm 0.06\) & \(0.758 \pm 0.02\) & \(0.771 \pm 0.04\) & \(0.897 \pm 0.07\) & - & - \\
Cycles + Trees & \(0.767 \pm 0.03\) & \(0.351 \pm 0.05\) & \(0.736 \pm 0.02\) & \(0.745 \pm 0.04\) & \(0.888 \pm 0.06\) & \(0.722 \pm 0.04\) & \(0.486 \pm 0.03\) \\
T. Cycles + T. Trees & \(0.789 \pm 0.03\) & \(0.463 \pm 0.07\) & \(0.760 \pm 0.02\) & \(0.773 \pm 0.04\) & \(0.891 \pm 0.06\) & - & - \\ 
\midrule
WL-PM & OUT-OF-MEM & \(0.577 \pm 0.08\) & \(0.886 \pm 0.01\) & \(0.853 \pm 0.04\) & \(0.756 \pm 0.05\) & - & - \\
WL-SP & \(0.757 \pm 0.04\) & \(0.282 \pm 0.10\) & \(0.823 \pm 0.02\) & \(0.614 \pm 0.03\) & \(0.719 \pm 0.02\) & \(0.559 \pm 0.12\) & \(0.396 \pm 0.07\) \\
WL-VH & \(0.789 \pm 0.05\) & \(0.532 \pm 0.12\) & \(0.840 \pm 0.01\) & \(0.850 \pm 0.02\) & \(0.755 \pm 0.03\) & \(0.725 \pm 0.05\) & \(0.508 \pm 0.03\) \\
SP & \(0.789 \pm 0.05\) & \(0.401 \pm 0.13\) & \(0.825 \pm 0.01\) & \(0.723 \pm 0.03\) & \(0.759 \pm 0.04\) & \(0.552 \pm 0.12\) & \(0.394 \pm 0.08\) \\
GHC-Tree \cite{nguyen2020graph} & - & - & \(0.893 \pm  0.08 \)  & - & - & \(0.721 \pm 0.03 \)  & \(0.486 \pm 0.04\) \\
GHC-Cycle \cite{nguyen2020graph} & - & - & \(0.878 \pm  0.07 \)  & - & - & \(0.709 \pm 0.05 \)  & \(0.474 \pm 0.04\) \\
\bottomrule
    \end{tabular}
    \end{adjustbox}
    \label{tab:comparison_graph}
\end{table}

\end{document}